\documentclass{article}
\usepackage{featureattcc}
\usepackage{amsfonts,amsmath,amssymb,amsthm,mathtools}
\usepackage{latexsym}

\usepackage[textwidth=30mm]{todonotes}

\newcommand{\Roberto}[1]{\todo[color=blue!20,author=\textbf{Roberto},inline]{\small #1\\}}

\usepackage{times}
\usepackage{soul}
\usepackage{url}
\usepackage[hidelinks]{hyperref}
\usepackage[utf8]{inputenc}
\usepackage[small]{caption}
\usepackage{graphicx}
\usepackage{amsmath}
\usepackage{amsthm}
\usepackage{booktabs}
\usepackage{algorithm}
\usepackage{algorithmic}
\usepackage[switch]{lineno}
\usepackage{comment}

\def\RR{{\mathbb R}}  
\def\PP{{\mathbb P}}

\def\EE{{\mathbb E}}

\newtheorem{example}{Example}
\newtheorem{theorem}{Theorem}
\newtheorem{lemma}[theorem]{Lemma}

\title{When is the Computation of a Feature Attribution Method Tractable?}

\author{
P. Barcel\'o$^{1,3,4}$
\And
R. Cominetti$^{1,2}$\And
M. Morgado$^{1}$\\
\affiliations
$^1$Institute for Mathematical and Computational Engineering, Universidad Cat\'olica de Chile\\
$^2$Department of Industrial and Systems Engineering, Universidad Cat\'olica de Chile\\
$^3$IMFD Chile\hspace{2ex}and\hspace{1ex}
$^4$CENIA Chile\\
\emails{
{pablo.barcelo, 
roberto.cominetti, mat.mw} @uc.cl}}

\begin{document}

\maketitle
\begin{abstract}
    Feature attribution methods have become essential for explaining machine learning models. Many popular approaches, such as SHAP and Banzhaf values, are grounded in power indices from cooperative game theory, which measure the contribution of features to model predictions. This work studies the computational complexity of power indices beyond SHAP, addressing the conditions under which they can be computed efficiently. 
    We identify a simple condition on power indices that ensures that computation is polynomially equivalent to evaluating expected values, extending known results for SHAP. We also introduce Bernoulli power indices, showing that their computation can be simplified to a constant number of expected value evaluations. Furthermore, we explore interaction power indices that quantify the importance of feature subsets, proving that their computation complexity mirrors that of individual features. 
\end{abstract}

\section{Introduction}

\paragraph{Feature attribution methods.} Explainability of machine learning models has become one of the most pressing challenges associated with deploying this technology in critical applications. In fact, in many domains such as healthcare, finance, and autonomous systems, the ability to understand and trust model predictions is essential for ensuring safety, fairness, and accountability. As a result, research on providing robust and meaningful explanations for machine learning models has proliferated in recent years 
\cite{ribeiro2016should,doshi-velez2017towards,lipton2018mythos,DBLP:conf/aaai/0001I22,DBLP:conf/lics/Darwiche23}.

A popular approach to explaining a machine learning model's prediction is through {\em feature attribution methods}. These methods aim to quantify the contribution of each individual feature to the prediction $F(\omega)$ made by a model $F$ for a given input $\omega$. Several of the feature attribution methods used in practice are based on power indices from cooperative game theory, including the well-known {\em Shapley} and {\em Banzhaf} values
\cite{lundberg2017unified,DBLP:conf/icml/SundararajanN20,DBLP:conf/uai/KarczmarzM0SW22,DBLP:conf/aistats/WangJ23}. We provide a brief introduction to such indices below.

\paragraph{Power indices.} 
Let $Y\!=\!(Y_1,\ldots,Y_n)$ be a vector of features where each $Y_i$ takes values in a finite set $\Omega_i$. We assume that the $Y_i$'s are independent random variables with corresponding product distribution $$\PP(Y\!\!=\!\omega) \ = \ \prod_{i=1}^n\PP(Y_i\!=\!\omega_i)$$
 over
the set of possible outcomes $\omega\in\Omega=\Omega_1 \times \ldots \times \Omega_n$. We are also given a learned 
model $F:\Omega\to\RR$ that computes a prediction $F(\omega)$, for each $\omega \in \Omega$. Throughout the paper, we assume that $F(\omega)$ can be computed in polynomial time in the number of features $n$. 

Consider a fixed $e \in \Omega$ and let $N\triangleq\{1,\ldots,n\}$. For each $a \in N$, we measure the relevance of the feature $Y_a$ in explaining the outcome $F(e)$ by using power indices from cooperative game theory. To do this, we first consider 
the conditional expectation of $F$ on the event  $\{Y_S\!=\!e_S\}$, that is
\begin{align*}
\EE[F|S]&=\sum_{\omega\in\Omega}F(\omega)\,\PP(Y\!=\!\omega|Y_S\!=\!e_S)\\
&=\sum_{\omega\in\Omega}F(\omega)\prod_{i\in S}\delta_i(\omega_i)\prod_{i\not\in S}\PP(Y_i\!=\!\omega_i)
\end{align*} 
with $\delta_i$ defined as follows: 
$$
\delta_i(w_i) \ = \ 
\begin{cases} 
1 \ \ \ &\text{if $w_i = e_i$}, \\ 
0 \ \ \ &\text{otherwise}. 
\end{cases} 
$$
The {\em marginal contribution} of feature $a\in N$ to a coalition $S\subseteq N \setminus \{a\}$ is  
defined as 
$$m(a;S) \ \triangleq \ \EE[F|S\cup\{a\}]-\EE[F|S].$$
Then, for a given probability distribution $Q$ over 
the set of all subsets of 
$N \setminus \{a\}$, 
the {\em average  contribution} of $a$ according to $Q$ is given by
$$I(a;F) \ \triangleq  \sum_{S\subseteq N \setminus \{a\}} \!\!\!Q(S)\,m(a;S).$$
In cooperative games this corresponds to a {\em power index}, among which the Shapley and Banzhaf values are the most prominent examples: 

\begin{itemize} 

\item {\em Shapley value:} In this case, $Q(S) = \frac{|S|! (n - |S| -1)!}{n!}$ for each $S \subseteq 
N \setminus \{a\}$. 

\item {\em Banzhaf value:} This value considers a uniform probability distribution over the subsets of 
$N \setminus \{a\}$, i.e., $Q(S) = \frac{1}{2^{n - 1}}$ for each $S \subseteq N \setminus \{a\}$. 
 
\end{itemize} 
Throughout the paper, we will also consider some less common alternatives of power indices, 
such as the binomial indices, the Bernoulli indices, and others.

\paragraph{Previous results.} The complexity of computing Shapley values, often referred to as SHAP scores in the context of ML models, has garnered significant attention in recent years \cite{van2022tractability,DBLP:journals/jmlr/ArenasBBM23,DBLP:conf/ijcai/Huang024}. 
In particular, \citeauthor{van2022tractability} have shown that the problem of computing SHAP scores over a model $F$ is polynomially equivalent to computing expected values over $F$. Hence: 

\begin{itemize}

\item The computation of SHAP scores is efficient for any class of models where expected values can be calculated in polynomial time. This means, e.g., that SHAP scores can be computed in polynomial time for linear regression models \cite{DBLP:conf/ijcai/KhosraviLCB19}, decision trees \cite{DBLP:journals/corr/abs-2006-16341}, 
deterministic and decomposable Boolean circuits \cite{DBLP:journals/jmlr/ArenasBBM23}, and CNF formulas of bounded treewidth \cite{DBLP:conf/lpar/FerraraPV05}. 

\item The computation of SHAP scores becomes intractable for model classes where computing expected values is itself intractable. Notably, this includes logistic regression models \cite{van2022tractability} and (monotone) DNF formulas \cite{DBLP:journals/siamcomp/ProvanB83}. 
\end{itemize}

\paragraph{Context.} Although SHAP scores are the most widely studied feature attribution method, it has recently been argued that they are by no means a {\em silver bullet} for solving practical explainability challenges. In fact, it has been noted that, in some cases, other indices, such as the Banzhaf values, can provide more meaningful explainability results, depending on the class of models considered \cite{DBLP:conf/uai/KarczmarzM0SW22,DBLP:conf/aistats/WangJ23}. Unfortunately, the computational complexity of power indices other than SHAP remains poorly understood, leading to the following natural question: 
\begin{itemize}
    \item Which power indices, as defined above, and for which classes of models, allow computation to be performed in polynomial time?
\end{itemize}
 This question is critical for assessing the practical applicability of different power indices. Based on the previous observations, we rephrase this question as follows: 
\begin{itemize} 
\item Which power indices enable computation that is polynomially equivalent to evaluating expected values? 
\end{itemize} 

\paragraph{Results on simple power indices.} We first focus on so called {\em simple} power indices, in which the value of $Q(S)$ only depends on the cardinality of $S$, for each $S \subseteq N \setminus \{a\}$. In this case, we denote $Q(S)$ as $q_{|S|}$. 
Notice that both Shapley and Banzhaf values 
are simple power indices. 

\begin{itemize} 

\item {\em Power index computation based on expected values:}
We begin by observing that a straightforward extension of the results by \citeauthor{van2022tractability} shows that the computation of any simple power index over a model $F$ can be reduced, in polynomial time, to the computation of expected values over $F$. Specifically, the computation of an arbitrary simple power index over $F$ can be carried out by evaluating the expected value of $F$ over $2n$ different probability distributions, where $n$ is the number of features of $F$, and then performing polynomial interpolation on the resulting values. 

\item {\em Expected value computation based on power indices:}
We then study the converse direction, i.e., whether the computation of expected values for $F$ can be carried out in polynomial time if we are granted access to an oracle that computes power indices for $F$ in polynomial time. We show that this is true for any simple power index that satisfies $q_0 > 0$, which includes many of the power indices studied in the literature (e.g., Shapley and Banzhaf values). While for Shapley values, this is a straightforward consequence of its {\em efficiency property}, for other indices that do not satisfy this property (including Banzhaf), the proof is more complicated and involves polynomial interpolation. We also observe that the condition $q_0 > 0$ is necessary, in the sense that there are simple 
power indices with $q_0 = 0$ and classes of models where the computation of the index is tractable, while the computation of expected values is intractable.  
\end{itemize}

\paragraph{Results on Bernoulli power indices.}
As previously mentioned, simple power indices can be computed via polynomial interpolation over a polynomial number of expected values. We prove that this can be further simplified for what we term \emph{Bernoulli} power indices. These indices are defined by probability distributions \( Q \) such that \( Q(S) \), for any \( S \subseteq N \setminus \{a\} \), is determined by independent Bernoulli trials. Notably, we show that for such power indices—which are not necessarily simple and include the important case of the Banzhaf value—computation can be reduced, in polynomial time, 
to calculating the difference between only \emph{two} expected values.

\paragraph{Results on interaction power indices.} 
A recent trend in feature attribution methods focuses on developing approaches that evaluate the importance not only of individual features, as exemplified by the previously introduced power indices, but also of feature sets of arbitrary size \cite{beliakov2020discrete,DBLP:conf/icml/SundararajanDA20}. This shift aims to capture and quantify the significance of feature interactions. However, the computational complexity of interaction power indices across different classes of models remains unexplored.
We show that for {\em simple} interaction power indices, the computational scenario mirrors that of simple power indices for individual features. Specifically, the computation of simple interaction power indices can be performed in polynomial time if and only if expected values can also be computed in polynomial time. Proving this result requires an extension of our previous techniques, as the computation of interaction power indices now relies on \emph{bivariate} polynomial interpolation using expected values. We also investigate \emph{Bernoulli} interaction power indices and show that, for feature sets of constant size, these indices can be computed with a fixed number of calls to an oracle for evaluating expected values.

\paragraph{Organization of the paper.} Our results on simple power indices are presented in Section \ref{sec:spi}, those for Bernoulli power indices in Section \ref{sec:bpi}, and those for interaction power indices in Section \ref{sec:ipi}. 
Final remarks can be found in Section \ref{sec:final}. 

\paragraph{Notations.} We let $N=\{1,\ldots,n\}$ be the set of features. For $A\subseteq N$ we denote $|A|$ its cardinality, $A^c=N\setminus A$ its complement, $\mathcal{P}(A)$ the family of all subsets of $A$, and $\mathcal{P}_k(A)$ the subsets of cardinality $k$.

\section{Cardinality-based indices}
\label{sec:spi}

We begin by considering the {\em simple power indices} where the probability $Q(S)\!=\!q_{|S|}$  depends only on the 
cardinality of $S$, and hence $q_0,\ldots,q_{n-1}\!\!\geq\! 0$ are such 
that $\sum_{k=0}^{n-1}{n-1\choose k}q_k\!=\!1$. 

\subsection{From expected values to power indices}

In this section, we show that the computation of any simple power index for a model 
$F$ can be reduced, in polynomial time, to evaluating expected values over 
$F$. This result builds on a straightforward extension of the techniques introduced by \citeauthor{van2022tractability}, who established a similar reduction specifically for SHAP scores. 

Denoting $\mathcal{P}^{a}_k=\mathcal{P}_k(\{a\}^c)$ the set of coalitions $S\subseteq\{a\}^c$ of size $|S|=k$, we have 
\begin{equation}\label{Eq:q-index}I(a;F)=\sum_{k=0}^{n-1}q_k\,m_{k}(a)\end{equation}
where
\begin{equation}\label{Eq:q-index2} m_{k}(a)=\!\!\!\sum_{S\in\mathcal{P}^{a}_k}\!\!m(a;S).\end{equation}
Thus, in order to compute  $I(a;F)$ it suffices to  calculate the sums $m_{k}(a)$. This can be done with the help of the following variant of \cite[Claim 2]{van2022tractability}.
\begin{lemma}\label{Le:fundamental}
 Let $G:\Omega\to\RR$ and $c_k=\sum_{S\in\mathcal{P}_k}\!\EE[G|S]$ with $\mathcal{P}_k$ the family of all subsets of $N$ of size $k$.
 For $z\geq 0$ let $Y^z=(Y_1^z,\ldots,Y_n^z)$ be independent random variables with  
$$\PP(Y^z_i\!=\!\omega_i) \ = \ 
\frac{z\,\delta_i(\omega_i)+\PP(Y_i\!=\!\omega_i)}{1+z}.$$ 
Then
$\sum_{k=0}^{n}c_kz^k=(1\!+\!z)^{n}\,\EE[G(Y^z)]$.
\end{lemma}
\begin{proof}
By direct substitution we have
     \begin{align*}
        \sum_{k=0}^{n}c_k z^k&=\sum_{k=0}^{n}\sum_{S\in\mathcal{P}_k}\!\EE[G|S]\cdot z^k \\
&=\sum_{S\subseteq N}\!\EE[G|S]\cdot z^{|S|} \\
&=\sum_{S\subseteq N}\!z^{|S|}\sum_{\omega\in\Omega}G(\omega)\prod_{i\in S}\delta_i(\omega_i)\prod_{i\not\in S}\PP(Y_i\!=\!\omega_i).
    \end{align*}
Exchanging the order of the sums we get
    \begin{align*}
        \sum_{k=0}^{n}c_k z^k
&=\sum_{\omega\in\Omega}G(\omega)\!\!\sum_{S\subseteq N}\prod_{i\in S}z\, \delta_i(\omega_i)\times\prod_{i\not\in S}\PP(Y_i\!=\!\omega_i).
    \end{align*}
Now, using the
    identity 
   \begin{equation}\label{Eq:mix}
   \sum_{S\subseteq N}\mbox{$\prod_{i\in S}u_i\times \prod_{i\in S^c}v_i$}=\prod_{i=1}^n(u_i+v_i)
   \end{equation}
    the inner sum above is exactly 
$(1+z)^n\PP(Y^z\!\!=\!w)$ and then
$\sum_{k=0}^{n}c_k z^k
=(1\!+\!z)^n\;\EE[G(Y^z)]$.
\end{proof}

Using Lemma \ref{Le:fundamental}, we prove that the computation of simple power indices can be reduced to evaluating expected values in polynomial time. For the Shapley value—defined by \( q_k = \frac{k!\,(n-1-k)!}{(n-1)!} \)—this result was previously established in \cite{van2022tractability}.

\begin{theorem}\label{Th:direct}
   The computation of a simple power index $I(a;F)$,
   for any feature $a \in N$, can be reduced in polynomial time to evaluate  $\EE[F(Z)]$ for a family of $2n$ random vectors $Z=(Z_1,\ldots,Z_n)$, each of them with a product distribution.
\end{theorem}
\begin{proof}
As mentioned earlier,  computing $I(a;F)$ boils down to calculate the coefficients $c_k=m_{k}(a)$ for $0\leq k\leq n\!-\!1$.
Recall that we have been given a fixed partial instance $e \in \Omega$. 
Let $\Omega^{-a}\!=\bigotimes_{i\neq a}\Omega_i$.  
For $\omega^{-a} \in \Omega^{-a}$, we slightly abuse notation and write $(e_a,\omega^{-a})$ for the element in $\Omega$ that is obtained by extending $\omega^{-a}$ with value $e_a$ for feature $a$. 
Analogously, we write $(Y_a,\omega^{-a})$ and $(Y_a,(Y^z_{i})_{i\neq a})$, where $Y_a$ is a random value that ranges over $\Omega_a$. 
 
Consider the map $G:\Omega^{-a}\to\RR$ 
     defined by $$G(\omega^{-a})\ = \ F(e_a,\omega^{-a})-\EE[F(Y_a,\omega^{-a})]$$ with a partial expectation of $F(\cdot)$ with respect to $Y_a$.
     Then $m(a;S)=\EE[G|S]$ so that
   $c_{k}=\sum_{S\in\mathcal{P}_k^{a}}\EE[G|S]$ and Lemma \ref{Le:fundamental} gives the following: 
$$\sum_{k=0}^{n-1}c_k\,z^k=(1\!+\!z)^{n-1}\EE[G((Y^z_{i})_{i\neq a})].$$

Evaluating at $n$ different points 
$z_0,z_1,\ldots,z_{n-1}\geq 0$, we obtain a linear 
system for the coefficients $c_0,\ldots,c_{n-1}$. This system is defined by a Vandermonde matrix and can be solved uniquely. Thus, the computation of $I(a;F)$ reduces to evaluate the expectations $\EE[G((Y^z_{i})_{i\neq a})]$  
for $z=z_0,\ldots,z_{n-1}$.
The proof is completed by noting that each such expectation for $G$ amounts to compute the difference between the 
expectations of $F$ for two product distributions, namely
$$\EE[G((Y^z_{i})_{i\neq a})]=\EE[F(Y^{1,z})]-\EE[F(Y^{0,z})]$$
with $Y^{1,z}\!=(e_a,(Y^z_{i})_{i\neq a})$ and $Y^{0,z}\!=(Y_a,(Y^z_{i})_{i\neq a})$.
\end{proof}

\subsection{From power indices to expected values} 

Consider a simple power index \( I(a; F) \) with \( q_0 > 0 \). We show that the evaluation of the expected value for a model \( F \) reduces, in polynomial time, to the computation of \( I(a; F) \) for every feature \( a \in N \). This property is well-known for the SHAP score, as it follows directly from the \emph{efficiency} property of the Shapley value, which ensures that 
\[
\sum_{a=1}^n I(a; F) =F(e)-\EE[F].
\] 
However, this identity does not hold for the Banzhaf value or other cardinality-based indices. Consequently, an alternative argument is required to establish the reduction in these cases.

Towards this end we first establish a technical Lemma,
for which we introduce the following notations.
Consider the random vectors $Y^z$ as in Lemma \ref{Le:fundamental}, and denote $\EE_z[\cdot]$  the expectation  for the corresponding product distribution, keeping $\EE[\cdot]$ for the original distribution of $Y$. Denote by
$m^z(a;S)$ and $I^z(a;F)$ 
the marginal contributions and indices computed with $Y^z$. Finally, for $k=0,\ldots,n$ let 
$$\beta_k^n \ = \ k\, q_{k-1}-(n-k)\,q_k$$ 
with $q_{-1}=q_n=0$, and consider the polynomials $$\mbox{$P_\ell(z) \ \triangleq \ \sum_{k=0}^{\ell}{\ell\choose k}$}\,\beta_k^n\,(1+z)^kz^{\ell-k}.$$ 
\begin{lemma}
    Let $\Theta(z)=\sum_{a=1}^nI^z(a;F)$ be the sum of all indices
     computed for $Y^z$.
     Then, setting  $c_\ell\triangleq\sum_{S\in\mathcal{P}_\ell}\EE[F|S]$ we have $\sum_{\ell=0}^nc_\ell\,P_\ell(z)=(1+z)^n\,\Theta(z)$.
\end{lemma} 
\begin{proof}
By direct substitution we have 
\begin{align*}
\Theta(z)&=\sum_{a=1}^n\sum_{k=0}^{n-1}\sum_{S\in\mathcal{P}^{a}_k}\!\!q_k\,m^{z}(a;S)\\
    &=\sum_{a=1}^n\sum_{k=0}^{n-1}\sum_{S\in\mathcal{P}^{a}_k}\!\!q_k\big(\EE_{z}[F|S\cup\{a\}]-\EE_{z}[F|S]\big).
\end{align*}
In this sum, each subset $A\subseteq N$ of size  $|A|=k$ appears $k$ times as a positive term $q_{k-1}\,\EE_{z}[F|S\cup\{a\}]$ for each $a\in A$, and $n-k$ times as a negative term $-q_k\,\EE_{z}[F|S]$ for $a\not\in A$, so that the sum can be expressed as
\begin{align*}
\Theta(z)&=\sum_{k=0}^{n}\sum_{A\in\mathcal{P}_k}\beta_k^n\,\EE_{z}[F|A].
\end{align*}
Now, using Lemma \ref{Le:fundamental} and denoting $\mathcal{P}_j(A^c)$ the subsets of $A^c$ of size $j$, we may  compute $\EE_{z}[F|A]$ as
$$\EE_{z}[F|A]=\frac{1}{(1
+z)^{n-k}}\sum_{j=0}^{n-k}z^j\!\!\!\!\sum_{B\in\mathcal{P}_j(A^c)}\!\!\!\!\!\EE[F|A\cup B]$$
which substituted above gives
\begin{align*}
\Theta(z)&=\sum_{k=0}^{n}\sum_{A\in\mathcal{P}_k}\sum_{j=0}^{n-k}\sum_{B\in\mathcal{P}_j( A^c)}\!\!\!\!\!\!\beta_k^n\,\frac{z^j}{(1+z)^{n-k}}\, \EE[F|A\cup B].
\end{align*}
Next, we observe that every set $S\subseteq N$ of size $|S|=\ell$ admits ${\ell\choose k}$ decompositions
as $S=A\cup B$ with $A\in\mathcal{P}_k$ and $B\in\mathcal{P}_{j}(A^c)$, where $j=\ell-k$, so that
 we may further rewrite 
\begin{align*}
\Theta(z)&=\sum_{\ell=0}^{n}\sum_{S\in\mathcal{P}_\ell}\!\EE[F|S]\sum_{k=0}^{\ell}\mbox{${\ell\choose k}\,\beta_k^n\,\frac{z^{\ell-k}}{(1
+z)^{n-k}}$}\\&=\mbox{$\frac{1}{(1+z)^n}$}\sum_{\ell=0}^nc_\ell\,P_\ell(z).
\end{align*}

\vspace{-4ex}
\end{proof}

Suppose we want to evaluate the expectation $c_0=\EE[F]$. Assuming that the 
simple power indices $I^z(a;F)$ can be computed efficiently, we may evaluate $(1\!+\!z)^n\,\Theta(z)$ at $n+1$ different points $z_0,\ldots,z_n$ 
in order  to write a linear system
$M c=b$ for the coefficients $c=(c_0,\ldots,c_n)$ with $M_{i,\ell}=P_\ell(z_i)$  
and $b_i=(1+z_i)^n\,\Theta(z_i)$ for $0\leq i,\ell\leq n$.
Unfortunately the family $\{P_\ell(\cdot):\ell=0,\ldots,n\}$ is linearly dependent  and the matrix $M$ turns out to be singular (see below). However, since we have explicitly $c_n=F(e)$, we may write a reduced system in 
the unknowns $(c_0,\ldots,c_{n-1})$. We next establish sufficient conditions under which this reduced system is nonsingular, by exploiting the following simple fact. 

\vspace{1ex}
{\sc Remark}. Let $P_0(\cdot),\ldots,P_{n-1}(\cdot)$ be polynomials with $\mathop{\rm deg}(P_\ell)=\ell$.
Then, for any family $z_0,\ldots,z_{n-1}\in\RR$ of different reals, the  following matrix is nonsingular
$$M=\left(
\begin{array}{ccc}
P_0(z_0)&\cdots&P_{n-1}(z_0)\\
\vdots&&\vdots\\
P_0(z_{n-1})&\cdots&P_{n-1}(z_{n-1})
\end{array}\right).$$
 Indeed, take $\alpha\in\mathop{\rm Ker}(M)$ so that $M\alpha=0$. This implies that  $P(x)=\sum_{j=0}^{n-1}\alpha_j\,P_j(x)$ is a polynomial of degree $n-1$ with $n$ distinct roots $z_0,\ldots,z_{n-1}$ and therefore $P(\cdot)\equiv 0$. Since $\mathop{\rm deg}(P_\ell)=\ell$, the polynomials $P_\ell(\cdot)$ are linearly independent so that $\alpha=0$ and therefore $M$ is nonsingular. \qed

 \vspace{1ex}
Using this, we get the following converse of Theorem \ref{Th:direct}.

\begin{theorem}\label{Th:converse}
Consider a simple power index such that $q_0>0$.
Then,
computing $\EE[F]$
reduces in polynomial time to evaluate 
the simple power 
indices $I^z(a;F)$ for all $a=1,\ldots,n$ at
$n$ different positive reals $z=z_0,\ldots,z_{n-1}\in\RR_+$.
\end{theorem}

\begin{proof}
Recall that $c_n=F(e)$ and consider the reduced linear system
in the unknowns $c=(c_0,\ldots,c_{n-1})$:
$$\sum_{\ell=0}^{n-1}c_\ell\,P_\ell(z_i)=(1+z_i)^n\,\Theta(z_i)-c_n\,P_n(z_i)\quad\forall\, 0\leq i\leq n-1.$$
The left hand side corresponds to  $M c$ with $M$ as in the remark above.
We claim that since $q_0>0$ each polynomial
$P_{\ell}(\cdot)$ has degree $\ell$ for $\ell=0,\ldots,n-1$. Indeed, the leading coefficient of $z^\ell$ in $P_\ell(z)$ is $a_\ell=\sum_{k=0}^{\ell}{\ell\choose k}\,\beta^n_k$. Recalling that $q_{-1}=0$ and after some simple algebraic manipulations, we get
\begin{align*}
   a_\ell&=\mbox{$\sum_{k=0}^{\ell}{\ell\choose k}(k\, q_{k-1}-(n-k)\,q_k)$}\\
     &=(\ell-n)\mbox{$\sum_{k=0}^{\ell}{\ell\choose k}\,q_k$}.
\end{align*}
The latter is strictly negative for $0\leq\ell\leq n-1$ so that $P_\ell(\cdot)$ has degree $\ell$. Invoking the previous remark it follows that
$M$ is nonsingular and then we can solve the linear system for $c$ in order to recover $c_0=\EE[F]$.
\end{proof}

\noindent {\sc Remark.} Observe that for $\ell=n$ the polynomial $P_n(
\cdot)$ has degree at most $n-1$, so that the family $\{P_\ell(\cdot):\ell=0,\ldots,n\}$ with $P_n(\cdot)$ included is always linearly dependent. \qed

\begin{example} \ 
{\em \begin{itemize}
    \item[(a)] \underline{\em Shapley value}:
For $q_k=\frac{k!\,(n-k-1)!}{n!}$ a direct calculation gives $P_\ell(z)=-z^\ell$ for $\ell=0,\ldots,n-1$ which are  linearly independent with $a_\ell=-1$.
    \item[(b)] \underline{\em Banzhaf value}:
With $q_k={1}/{2^{n-1}}$ for $0\leq k\leq n\!-\!1$ we have
$P_\ell(z)\!=\!\mbox{$\frac{1}{2^{n-1}}$}\big(\ell (1+2z)^{\ell-1}\!-(n-\ell)(1+2z)^\ell\big)$
which has degree $\ell$ with $a_\ell=(\ell-n) 2^{\ell-n+1}$. 
    \item[(c)] \underline{\em  Binomial value}: Binomial indices correspond to the choice
$q_k=\theta^k(1-\theta)^{n-1-k}$ with $\theta\in(0,1)$.
In this case $P_\ell(z)=(1-\theta)^{n-\ell-1}\big(\ell(z+1)(z+\theta)^{\ell-1}-n(z+\theta)^\ell\big)$
with $a_\ell=(\ell-n)(1-\theta)^{n-\ell-1}$.
    \item[(d)] \underline{\em Dictatorial value}: This is given by
$q_{0}=1$ and $q_k=0$ for $k=1,\ldots,n-1$. Here $p_\ell(z)=\ell z^{\ell-1}-(n-\ell)z^\ell$
 and $a_\ell=\ell-n$.
    \item[(e)] \underline{\em  Marginal value}: In this case
$q_{n-1}=1$ and $q_k=0$ for $k=0,\ldots,n-2$. Linear independence fails since $P_\ell(z)\equiv 0$ for $\ell=0,\ldots,n-2$. 
Moreover, even if $\EE[F]$ is difficult to compute, the marginal indices can be computed efficiently (assuming that $F$ is) by using the explicit formula
\begin{align*}
    I(a;F)&=\EE\big[F|N\big]-\EE\big[F|N\setminus\{a\}\big]\\
    &=F(e)-\!\!\sum_{\omega_a\in\Omega_a}\!\!F(\omega_a,e_{-a})\PP(Y_a\!=\!\omega_a).
    \end{align*}
This shows that in general there is no polynomial time reduction from the indices to the expectation, and some additional condition such as $q_0>0$ is needed. \qed
\end{itemize} } 
\end{example} 


\section{Bernoulli indices}
\label{sec:bpi} 

As we show in this section, 
some power indices can be computed much more efficiently than through polynomial interpolation over a polynomial number of expected values. 
Let $a\in N$ be a fixed feature and consider a power index in which the probability $Q$ over $\mathcal{P}(\{a\}^c)$ results from independent Bernoulli trials. Specifically, we consider a random subset $S\subseteq\{a\}^c$ that includes each $i\neq a$ with probabilities $\theta_i\in[0,1]$, so that  $$Q(S)=\prod_{i\in S}\theta_i\cdot\prod_{i\not\in S\cup\{a\}}(1-\theta_i).$$ 
The corresponding index 
$$I(a;F) \ \triangleq \ \sum_{S\subseteq\{a\}^c}\!\!\!Q(S)\big(\EE\big[F|S\cup\{a\}\big]-\EE\big[F|S\big]\big)$$
is then called a {\em Bernoulli} power index. 

\medskip 
\noindent{\sc Remark.} This class includes the binomial indices with $\theta_i \equiv \theta$ and the Banzhaf value with $\theta_i \equiv \frac{1}{2}$, both of which are also cardinality-based indices. However, apart from these cases, Bernoulli indices are not necessarily simple, as $Q(S)$ does not need to depend solely on $|S|$. Observe that the SHAP score is not a Bernoulli index.
  \qed
\vspace{2ex}

Suppose that the feature $a$ is also included with some
probability $\theta_a$ and let
$$Q_\theta(S) \ \triangleq \ 
\prod_{i\in S}\theta_i\cdot\prod_{i\not\in S}(1-\theta_i)$$ be the distribution over all subsets $S\subseteq N$. By considering $\theta^1$ and $\theta^0$
the vectors where we take respectively $\theta_a=1$ and $\theta_a=0$, and denoting $Q_1$ and $Q_0$ the corresponding distributions over $\mathcal{P}(N)$, the Bernoulli index can be expressed  as
$$I(a;F) \ = \ \sum_{S\subseteq N}\!Q_{1}(S)\,\EE\big[F|S\big]-\sum_{S\subseteq N}\!Q_{0}(S)\,\EE\big[F|S\big].$$
We claim that this is a difference of expectations with respect to two product distributions, namely 
\begin{theorem}\label{Thm:cinco}
  Let $Y^\theta=(Y^{\theta_1}_1,\ldots,Y^{\theta_n}_n)$
be a tuple of independent random variables with marginals given by the mixtures
$\PP(Y_i^{\theta_i}\!=\!\omega_i)
=\theta_i \delta_i(\omega_i)+(1\!-\!\theta_i)\PP(Y_i\!=\!\omega_i)$. Then, denoting $Y^{\theta,0}$ and $Y^{\theta,1}$ the tuples $Y^\theta$ with $\theta_a\!=\!0$ and $\theta_a\!=\!1$, we have $$I(a;F) \ = \ \EE\big[F(Y^{\theta,1})\big]-\EE\big[F(Y^{\theta,0})\big].$$
\end{theorem}
\begin{proof}
This follows from the general identity 
\begin{eqnarray*}
  && \hspace{-5ex}\sum_{S\subseteq N}\!\!Q_\theta(S)\,\EE[F|S]\\
  &&\hspace{-2ex} =\sum_{S\subseteq N}\!\!Q_\theta(S)\sum_{\omega\in\Omega}\!F(\omega)\prod_{i\in S}\delta_i(\omega_i)\!\times\!\prod_{i\not\in S}\PP(Y_{i}\!=\!\omega_{i})\\
&&\hspace{-2ex}=\sum_{\omega\in\Omega}\!F(\omega)\!\!\sum_{S\subseteq N}\prod_{i\in S}\theta_i\delta_i(\omega_i)\!\times\!\prod_{i\not\in S}(1\!-\!\theta_i)\PP(Y_{i}\!=\!\omega_{i})\\
&&\hspace{-2ex}=\sum_{\omega\in\Omega}F(\omega)\!\prod_{i\in N}\!\big(\theta_i\delta_i(\omega_i)+(1\!-\!\theta_i)\PP(Y_{i}\!=\!\omega_{i})\big)\\
&&\hspace{-2ex}=\sum_{\omega\in\Omega}F(\omega)\,\PP(Y^\theta\!\!=\!\omega)\\
&&\hspace{-2ex}=\EE[F(Y^\theta)].
\end{eqnarray*}  

\vspace{-4ex}
\end{proof}

\noindent{\sc Remark.} As mentioned earlier, the Banzhaf and binomial indices are Bernoulli indices, so by Theorem \ref{Thm:cinco} they can be calculated by computing just 2 expectations.
This is more efficient than considering them as cardinality-based and using Theorem \ref{Th:direct} which requires
 $2n$ expectations. \qed


\section{Interaction Indices}
\label{sec:ipi} 
 {\em Interaction indices} aim not only to measure the relevance of a feature $a \in N$ in isolation,  but more generally to explain  a subset of features $A\subseteq N$. An interaction index is defined as 
\begin{align*}
I(A;F) \ = \ \sum_{S\subseteq A^c} \!\!Q_A(S)\,m(A;S)
\end{align*}
where $Q_A(\cdot)$ is a probability on  $\mathcal{P}(A^c)$, and 
the marginal contribution of $A$ to $S$ is defined as in \cite{beliakov2020discrete,DBLP:conf/icml/SundararajanN20} by
$$m(A;S) = \sum_{B\subseteq A}(-1)^{|A\setminus B|}\EE\big[F|S\cup B\big].$$
Alternatively, we have $$I(A;F)\ =\  \sum_{B\subseteq A}(-1)^{|A\setminus B|}R(B),$$ 
where $R(B)=\sum_{S\subseteq A^c} Q_A(S)\,\EE\big[F|S\cup B\big]$.

\subsection{Cardinality-based interaction indices}

An interaction index is said to be {\em simple} if it is characterized by a distribution \( Q_A(S) = q(|S|, |A|) \), where the function \( q \) depends solely on the cardinalities of the sets \( S \) and \( A \). We show that the computation of simple interaction indices 
retains the same properties as their single-feature counterparts: 
\begin{itemize}
    \item The computation of any interaction index for a model \( F \) can be reduced in polynomial time to the computation of expected values over \( F \). 
    \item Conversely, if \( q(1,0) > 0 \), the computation of expected values over \( F \) can be reduced in polynomial time to the computation of interaction indices over \( F \).
\end{itemize}
Note that it suffices to prove the first item, as the second follows directly from Theorem~\ref{Th:converse}. This is because any index of the form \( I(a; F) \), where \( a \in N \), is also an interaction index. 
The proof of the first item (Theorem~\ref{theo:redext} below) differs from that of Theorem~\ref{Th:direct} due to the nature of interaction indices. Specifically, the computation now requires the use of {\em bivariate} polynomial interpolation to determine the indices based on a set of expected values.

\begin{theorem} \label{theo:redext}
Let $A\subseteq N$ with $|A|=m$. Then a simple interaction index $I(A;F)$ can be reduced  to computing   the expected values $\EE\big[F(Z)\big]$
for $(n\!-\!m\!+\!1)(m\!+\!1)$ random vectors $Z=(Z_1,\ldots,Z_n)$ with product distribution.
\end{theorem}


\begin{proof}
By splitting the sums according to the cardinality of the sets, and denoting  
$$c_{k,j} \ =\ \sum_{\stackrel{\mbox{\scriptsize $S\!\subseteq\!A^c$}}{
|S|=k}
}\sum_{\stackrel{\mbox{\scriptsize $B\!\subseteq\!A$}}{
|B|=j}
}\!\EE\big[F|S\cup B\big]$$
we have
\begin{align*}
I(A;F) &= \sum_{S\subseteq A^c}Q_A(S)\sum_{B\subseteq A}(-1)^{|A\setminus B|}\,\EE\big[F|S\cup B\big]\\
&= \sum_{k = 0}^{n-m}\sum_{j = 0}^{m}\,q(k,m)(-1)^{m - j}\,c_{k,j}
\end{align*}
so that $I(A;F)$ reduces to computing the  sums $c_{k,j}$. 

Consider the bivariate polynomial $$P(z,y) \ = \ \sum_{k = 0}^{n-m}\sum_{j = 0}^{m}c_{k,j}\,z^ky^j.$$
Proceeding as in the proof of Lemma \ref{Le:fundamental}, we have
     \begin{align}
P(z,y)
&=\!\sum_{\omega\in\Omega}F(\omega)\!\sum_{S\subseteq A^c}\!\sum_{B\subseteq A}\pi_{S,B}\label{Eq:ppp}
\end{align}
with
\begin{align*}
\pi_{S,B}&=\prod_{i\in S}z\,\delta_i(\omega_i)\prod_{i\in  B}y\,\delta_i(\omega_i)\!\!\prod_{i\not\in S\cup B}\!\!\PP(Y_i\!=\!\omega_i).
    \end{align*}
    The latter can be factorized as $\pi_{S,B}=\pi_S\cdot \pi_B$ with \begin{align*}  
    \pi_S&=\prod_{i\in S}z\,\delta_i(\omega_i)\!\!\prod_{i\in A^c\setminus S}\!\!\!\!\PP(Y_i\!=\!\omega_i)\\
    \pi_B&=\prod_{i\in B}y\,\delta_i(\omega_i)\!\!\prod_{i\in A\setminus B}\!\!\!\PP(Y_i\!=\!\omega_i)
    \end{align*}
    from which it follows that
    \begin{align*}
        \sum_{S\subseteq A^c}\!\sum_{B\subseteq A}\pi_{S,B}&=\mbox{$\left(\sum_{S\subseteq A^c}\pi_S\right)\left(\sum_{B\subseteq A}\pi_{B}\right)$}.
    \end{align*}
    Now, using the identity \eqref{Eq:mix} (with $A^c$ instead of $N$) we get
     \begin{align*}
        \mbox{$\sum_{S\subseteq A^c}\pi_S$}&=\prod_{i\in A^c}\big(z\,\delta_i(\omega_i)+\PP(Y_i\!=\!\omega_i)\big)\\
        &=(1+z)^{|A^c|}\prod_{i\in A^c}\PP(Y^z_i\!=\!\omega_i)
    \end{align*}
    and similarly 
     \begin{align*}
        \mbox{$\sum_{B\subseteq A}\pi_B$}&=(1+z)^{|A|}\prod_{i\in A}\PP(Y^y_i\!=\!\omega_i).
    \end{align*}
    Combining these two expressions we deduce that
    \begin{align*}
        \sum_{S\subseteq A^c}\!\sum_{B\subseteq A}\pi_{S,B}&=(1+z)^n\PP(Z\!=\!\omega)
    \end{align*}
    with $Z=((Y^y_i)_{i\in A},(Y^z_i)_{i\in A^c})$.
    Substituting this into \eqref{Eq:ppp} it follows that 
    computing $P(z,y)$ reduces to evaluate an expected value $P(z,y)=(1+z)^n\EE\big[F(Z)\big]$ for a random vector $Z$ with product distribution.

    Now, evaluating $P(z,y)$ over a rectangular grid
    $$G=\{z_0,\ldots,z_{n-m}\}\times\{y_0,\ldots,y_m\}$$ with all the $z_i$'s and $y_j$'s different, we obtain a linear system for the coefficients $c_{k,j}$. The matrix for this system is a bivariate Vandermonde matrix $M$ with
    nonzero determinant (see \cite[page 23]{DeMarchi2014}): 
    $$|\mathop{\rm det}(M)|=\big(\prod_{i<k}|z_i-z_k|\big)^{m+1}\times \big(\prod_{j<\ell}|y_j-y_\ell|\big)^{n-m+1}\neq 0$$
    where the first product is over all pairs $i,k\in\{0,\ldots,n-m\}$ such that $i<k$, and the second over all $j,\ell\in\{0,\ldots,m\}$ with $j<\ell$.
Hence one can uniquely solve the system to find the coefficients $c_{k,j}$, and then compute $I(A;F)$.
\end{proof}

\subsection{Bernoulli interaction indices}

Following with the ideas developed in Section \ref{sec:bpi}, let us now consider a Bernoulli interaction index defined as follows for every 
$A \subseteq N$: 
$$Q_A(S)=\prod_{i\in S}\theta_i \times \!\!\!\prod_{i\in A^c\setminus S}\!\!(1\!-\!\theta_i).$$

\noindent{\sc Remark.} Notice that for any specific set $A$ we could choose different Bernoulli probabilities $\theta_i^A$ for $i\in A^c$. 
Since here we consider a fixed $A$, we keep the simpler notation $\theta_i$. \qed

\medskip 

To conclude, we present the following straightforward extension of Theorem \ref{Thm:cinco}. This result establishes that if \( A \) has size \( m \), then \( I(A;F) \) can be computed using an explicit formula involving \( 2^m \) expected values over \( F \). When \( m = 1 \), this reduces precisely to the statement of Theorem \ref{Thm:cinco}. Notice the difference with Theorem \ref{theo:redext}, as in this case the number of expected values one needs to compute does not depend on the number of features of the underlying model $F$.


\begin{theorem}
A Bernoulli interaction index $I(A;F)$ can be reduced  to computing   the expected values $\EE\big[F(Z)\big]$
for $2^m$ random vectors $Z=(Z_1,\ldots,Z_n)$ with product distribution. 
\end{theorem}

\begin{proof}
Using $I(A;F)=\sum_{B\subseteq A}(-1)^{|A\setminus B|}R(B)$, it suffices  to show that each $R(B)$ is an expectation $\EE\big[F(Z)\big]$ for some $Z$ with product distribution. Indeed, since
\begin{align*}
    R(B)
    &=\!\!\sum_{S\subseteq A^c} \EE\big[F|S\cup B\big]\prod_{i\in S}\theta_i \times \!\!\!\!\prod_{i\in A^c\setminus S}\!\!(1\!-\!\theta_i)\\
   \EE\big[F|S\cup B\big]&=\!\sum_{\omega\in \Omega}F(\omega)\!\!\prod_{i\in S\cup B}\!\!\!\delta_i(\omega_i)\times\!\!\!\prod_{i\notin S\cup B}\!\!\!\PP(Y_i\!=\!\omega_i)   
\end{align*}
after some straightforward algebraic manipulations we get $$R(B)=\!\sum_{\omega\in \Omega}F(\omega)\prod_{i\in B}\delta_i(\omega_i)\times\!\!\!\prod_{i\in A\setminus B}\!\!\!\PP(Y_i\!=\!\omega_i)\times D(\omega)$$ where
\begin{align*}
    D(\omega)&=\!\!\sum_{S\subseteq A^c} \prod_{i\in S}\theta_i\delta_i(\omega_i)\times \!\!\!\!\prod_{i\in A^c\setminus S}\!\!(1\!-\!\theta_i)\PP(Y_i\!=\!\omega_i)  \\
    &=\!\prod_{i\in A^c}\!\big(\theta_i\delta_i(\omega_i)+(1\!-\!\theta_i)\PP(Y_i\!=\!\omega_i) \big).
\end{align*}
Hence, denoting $Y_i^\theta$ a random variable  with distribution $$\PP(Y_i^\theta\!=\!\omega_i)
=\theta_i \delta_i(\omega_i)+(1\!-\!\theta_i)\PP(Y_i\!=\!\omega_i)$$
it follows that $R(B)=\EE\big[F(Z)\big]$ where
$$Z=((e_i)_{i\in B},(Y_i)_{i\in A\setminus B},(Y_i^\theta)_{i\in A^c}).$$ 

\vspace{-2ex}
\end{proof}

\section{Final remarks}
\label{sec:final}

We have provided a detailed analysis of the complexity of computing power indices, including scenarios where sets of features of arbitrary size are considered. As shown, the polynomial equivalence between computing the index and computing expected values is not unique to SHAP,  
but in some instances 
computing the index might actually be 
simpler than computing expected values.
An open question remains regarding a precise characterization of indices that exhibit this property. While $q_0 > 0$ guarantees this behavior, it is possible that certain indices with 
$q_0 = 0$ might also satisfy it.  

\bibliographystyle{featureattcc}
\bibliography{featureattcc.bib}

\end{document}